\documentclass[conference]{IEEEtran}
\usepackage[latin9]{inputenc}
\usepackage{times}
\usepackage{epsfig}
\usepackage{graphicx}
\usepackage{amsmath}
\usepackage{amssymb}

\usepackage[T1]{fontenc}    
\usepackage{url}            
\usepackage{subfigure}
\usepackage{booktabs}       
\usepackage{amsfonts}       
\usepackage{nicefrac}       
\usepackage{microtype}      
\usepackage{amsmath,amssymb,amsthm,dsfont,stmaryrd,amsopn,amsfonts,amssymb}
\usepackage{graphicx}
\usepackage{multirow}
 \usepackage{xr}
 \usepackage{xcolor}
 \usepackage{algorithm, algorithmic}

\def\x{{\mathbf x}}
\def\m{{\mathbf m}}
\def\s{{\mathbf s}}

\def\w{{\mathbf w}}

\def\e{{\mathbf e}}

\def\B{\mathbf{B}}

\def\W{\mathbf{W}}

\def\D{\mathbf{D}}

\def\0{\mathbf{0}}

\def\R{\mathds{R}}

\newtheorem{theorem}{Theorem}[section]
\newtheorem{proposition}{Proposition}[section]

\newtheorem{lemma}[theorem]{Lemma}

\DeclareMathOperator*{\argmax}{arg\,max}

\makeatletter

\def\ps@IEEEtitlepagestyle{%
  \def\@oddfoot{\mycopyrightnotice}%
  \def\@evenfoot{}%
}
\def\mycopyrightnotice{%
  {\footnotesize xxx-x-xxxx-xxxx-x/xx/~\copyright~2021 IEEE\hfill}
  \gdef\mycopyrightnotice{}
}

\@ifundefined{showcaptionsetup}{}{
 \PassOptionsToPackage{caption=false}{subfig}}
\usepackage{subfig}
\makeatother

\usepackage{eso-pic}
\newcommand\AtPageUpperMyright[1]{\AtPageUpperLeft{%
 \put(\LenToUnit{0.5\paperwidth},\LenToUnit{-1cm}){%
     \parbox{0.5\textwidth}{\raggedleft\fontsize{9}{11}\selectfont #1}}%
 }}%
\newcommand{\conf}[1]{%
\AddToShipoutPictureBG*{%
\AtPageUpperMyright{#1}
}
}

\title{Learning from Incomplete Features by Simultaneous Training of Neural Networks and Sparse Coding}
\conf{Accepted for presentation at L2ID Workshop at CVPR 2021 \\19 - 25 June, 2021} 

\begin{document}

\author{\IEEEauthorblockN{Cesar F.~Caiafa\IEEEauthorrefmark{1}\IEEEauthorrefmark{2}, Ziyao Wang\IEEEauthorrefmark{3}, Jordi Sol\'{e}-Casals\IEEEauthorrefmark{4}, and Qibin Zhao\IEEEauthorrefmark{1}}
\IEEEauthorblockA{
Tensor Learning Team, Center for Advanced Intelligence Project, RIKEN, JAPAN\IEEEauthorrefmark{1}\\
Instituto Argentino de Radioastronom\'{i}a, CONICET-CCT La Plata / CIC-PBA / UNLP, ARGENTINA\IEEEauthorrefmark{2}\\
School of Automation, Southeast University, CHINA\IEEEauthorrefmark{2}\IEEEauthorrefmark{3}\\
University of Vic - Central University of Catalonia, SPAIN\IEEEauthorrefmark{4}\\
Corresponding author: C. F. Caiafa (Email: ccaiafa@gmail.com)}  
}

\maketitle
\begin{abstract}
In this paper, the problem of training a classifier on a dataset with incomplete features is addressed. We assume that different subsets of features (random or structured) are available at each data instance. This situation typically occurs in the applications when not all the features are collected for every data sample.
A new supervised learning method is developed to train a general classifier, such as a logistic regression or a deep neural network, using only a subset of features per sample, while assuming sparse representations of data vectors on an unknown dictionary. 
Sufficient conditions are identified, such that, if it is possible to train a classifier on incomplete observations so that their reconstructions are well separated by a hyperplane, then the same classifier also correctly separates the original (unobserved) data samples. 
Extensive simulation results on synthetic and well-known datasets are presented that validate our theoretical findings and demonstrate the effectiveness of the proposed method compared to traditional data imputation approaches and one state-of-the-art algorithm.
\end{abstract}

\section{Introduction}

\label{sec:intro}
Learning methods from limited or imperfect data has attracted great attention in the literature recently. Datasets with limited, weak, noisy labels or incomplete features represent an important and still open problem. In this paper, we address the problem of training a classifier on a dataset with incomplete features, which arises in many machine learning applications where sometimes the measurements are incomplete, noisy or affected by artifacts. Examples of this situation include: recommendation systems built upon the information gathered by different users where not all the users have fully completed their forms; medical datasets where typically not all tests can be performed on every patient; or a self-driving vehicle or robot where objects in the view field can be partially occluded. 

Handling correctly the incomplete-features problem is a classical challenge in machine learning. Skipping missing features by setting them to zero values damages the classification accuracy \cite{Chechik:2008vk}.  Most previous studies addressed this problem by using an imputation approach, which consists of performing data completion followed by training the classifier with those reconstructions (referred here as the sequential method). However, this strategy cannot ensure the statistical consistence of the classifier, as data completion is usually fully unsupervised or label information is partially or inefficiently exploited.

In this work, a new supervised learning method is developed to train a general classifier, such as a logistic regression or a deep neural network, using only a subset of features per sample, while assuming sparse representations of data vectors on an unknown dictionary. The proposed method simultaneously learns the classifier, the dictionary and the corresponding sparse representation of each input data sample. In this way, we combine the approximation power and simplicity of sparse coding with the extraordinary ability of neural networks (NNs) to model complex decision functions (classifiers) with the goal to successfully train a classifier based on incomplete features. 

We analyze the limitations of the sequential approach (section \ref{sec:limitation}), {\it i.e.} imputation followed by training, and introduce the simultaneous classification and coding approach in section \ref{sec:gen_approach}. Our method consists of incorporating a sparse data representation model into a single cost function that is optimized for training the classifier and, at the same time, finding the best representation of the observed data. A learning algorithm is presented in section \ref{sec:algo} to train a classifier on incomplete features and  sufficient conditions under which such a classifier performs as good as the ideal classifier, {\it i.e.} the one that can be obtained from complete observations, is identified (section \ref{sec:theory}). Extensive experimental results are presented in section \ref{sec:experiments}, using synthetic and well known benchmark datasets that validate our theoretical findings and illustrate the effectiveness of the proposed method.

\subsection{Related work}  
Practical sequential methods based on statistical imputation, such as computing the ``mean'', ``regression'' and ``multiple'' imputation techniques are common practice \cite{RUBIN:2014vw}. Remarkably, it was shown that imputing with a constant, {\it e.g.} the mean, is Bayes-risk consistent only when missing features are not informative \cite{2019arXiv190206931J}. More elaborated completion methods were also explored, such as $K$-nearest neighbor estimators, multilayer  or recurrent NNs and others, see \cite{GarciaLaencina:2009in} and references therein. 

However, sequential methods do not fully exploit label information. Data labels can provide valuable information about missing features that could potentially improve the classifier learning process. Recent advances on probabilistic generative models have allowed for a formulation of supervised learning with incomplete features as a statistical inference problem arriving at algorithms that significantly outperformed sequential methods. In the seminal work \cite{Ghahramani:vv}, a framework for maximum likelihood density estimation based on mixtures models was proposed and successfully applied to small incomplete-features problems. In particular, a Gaussian Mixture Model (GMM) was fitted to incomplete-data through an Expectation-Maximization (EM) algorithm. Building upon this generative model strategy, some approaches have considered integrating out the missing values based on a simple logistic regression function \cite{Williams:2005kq,Bhattacharyya:2004ta}. Other versions of this approach proposed an explicit simultaneous learning of the model and the decision function \cite{Liao:2007fd,Dick:2008fj}.
While probabilistic generative models provided a nice and elegant approach to the incomplete-features problem showing good results on small datasets, they are not suitable for many modern machine learning applications because: (1) despite some acceleration techniques were explored, {\it e.g.} \cite{Lin:2006bu,2012arXiv1209.0521D}, those algorithms are computationally expensive becoming prohibitive for moderate to large datasets; (2) GMM is impractical for modeling high-dimensional datasets because the number of parameters to achieve good approximations becomes unmanageable; and (3) they do not consider complex classification functions as the ones provided by deep NN architectures.

Recently, some approaches based on the low-rank property of the features data matrix were investigated and algorithms for data completion were proposed incorporating the label information \cite{Goldberg:2010wg,Hazan:2015tt,2018arXiv180205380H}. Since the rank estimation of a matrix is a computationally expensive task, usually based on the Singular Value Decomposition (SVD), the obtained algorithms are prohibitive to solve modern machine learning problems with large datasets. Additionally, as in the case of the probabilistic generative models, none of these methods considered complex classification functions. To overcome this drawback, more recently, a framework based on various NN architectures such as autoencoders, multilayer perceptrons and Radial Basis Function Networks (RBFNs), was proposed for handling missing input data by setting a probabilistic model, {\it e.g.} a GMM, for every missing feature, which is trained together with the NN weights \cite{Smieja:2018te}. This method combined the great capability of NNs to approximate complex decision functions with the nice formulation of the GMM to model missing data. However, it inherited the drawbacks of GMMs, {\it i.e.} they are not well suited to higher-dimensional datasets.

On the other hand, during the last few years in the signal processing community, there has been a rapid development of theory and algorithms for sparse coding approximations which, by exploiting the redundancy of natural signals, are able to provide simple and accurate models of complex data distributions, see \cite{Lee:2007wn,Mairal:2009ku,LeCun:2010tj,Mallat:2009wr,Eldar:2012wf} and references therein. Sparse coding is nearly ubiquitous in Nature, for example, it is found in the way that neurons encode sensory information \cite{1996Natur.381..607O,Olshausen:1997p405}. Sparse representations of data showed to be useful also  in classification problems. In \cite{Huang:2006ud}, a Linear Discriminant Analysis (LDA) classifier was trained on  corrupted data providing a robust classification method. In \cite{MairalNIPS2008}, algorithms for learning {\it discriminative} sparse models, instead of purely {\it reconstructive} ones, were proposed based on simple linear and bilinear classifiers. Similar methods were also studied by either using class-specific dictionaries \cite{Ramirez:2010jm,Sprechmann:2010jt} or using a single one for all classes \cite{2011ISPM...28...27T}. However, these proposed methods neither were applied to the incomplete-features problem nor considered deep NN classifiers. 

\subsection{Problem formulation}
\label{sec:definitions}
We assume a supervised learning scenario with vector samples and labels $\{\x_i, y_i\}$, $i=1,2,\dots, I$, $\x_i \in \R^N$ and $y_i \in \{0,1,\dots, C-1\}$ ($C$ classes). However, we are constrained to observe only subsets of features and their labels:  $\{\x^o_i, y_i\}$, $\x^o_i \in \R^{M_i}$ with $M_i < N$. Unobserved (missing) features are denoted by $\x^m_i \in \R^{N-M_i}$. We consider arbitrary patterns of missing features, which are allowed to be different for each data instance $i$. The set of  indices of missing features at sample $i$ is denoted by ${\cal{M}}_i$, {\it i.e.} $\x^m_i = \x_i({\cal{M}}_i)$ and $\x^o_i = \x_i\left({\overline{\mathcal{M}}}_i\right)$.

We define the set of all $K$-sparse vectors $\Sigma_K^P = \{ \s \in \R^P \text{ s.t. } \|\s\|_0 \le K\}$ (containing at most $K$ non-zero entries) and assume that data vectors $\x_i$ admit $K$-sparse representations over an unknown dictionary $\D \in \R^{N\times P}$ ($P \ge N$): 
\begin{equation}\label{eq:sparserep}
\x_i = \D\s_i, \mbox{ with } \s_i \in \Sigma_K^{P}.
\end{equation} 
The columns of a dictionary are called ``atoms'' because every data vector can be written as a linear combination of at most $K$ elementary components. Sometimes dictionaries are orthogonal such as the ones derived from the Discrete Cosine or Wavelet \cite{Mallat:2009wr} transforms. However, overcomplete ($P \ge N$) nonorthogonal dictionaries have demonstrated to play an important role in image processing tasks such as denoising, inpainting, etc \cite{Mairal:2009ku,Jenatton:2010ua}.

By partitioning $\D$ according to the pattern of missing features at sample $i$, we obtain $\D^o_i =\D\left({\overline{\mathcal{M}}}_i,:\right) \in \R^{M_i\times P}$ and $\D^m_i =\D({\cal{M}}_i,:)\in \R^{{(N-M_i)}\times P}$, which according to equation (\ref{eq:sparserep}) implies:
\begin{equation}
\x^o_i = \D^o_i \s_i, \mbox{ and } \x^m_i = \D^m_i \s_i. 
\end{equation}

Let us assume that a perfect classifier, {\it e.g.} a logistic regression or deep NN, that assigns probability $p_{\Theta}(\hat{y} | \x)$ to predicted label $\hat{y}$ given data $\x$ can be trained on the complete dataset $\{\x_i, y_i\}$, such that, in a two-classes scenario ($C=2$),  $p_{\Theta}(\hat{y} = y_i | \x_i) > p_{\Theta}(\hat{y} \neq y_i | \x_i)$,  $\forall i =1,2,\dots, I$, where $\Theta$ is the set of trained parameters. Our goal is to develop a method to obtain an estimate $\hat{\Theta}$ of parameters using only the incomplete dataset $\{\x^o_i, y_i\}$ and to identify conditions under which such a classifier is compatible with the ideal one.


\subsection{Why training after imputation is difficult?} \label{sec:limitation}
If the $K$-sparse representations of the observations $\x^o_i$ were unique, then $\x_i$ can be perfectly reconstructed from the incomplete observations and the classifier can be successfully trained using these reconstructions. In the particular case where the dictionary is known in advance, there exist conditions on the sampling patterns based on the {\it coherence}, {\it spark} or {\it RIP (Restricted Isometry Property)} of matrix $\D^o_i$ that can guarantee uniqueness \cite{Eldar:2012wf}. However, these conditions are difficult to meet in practice and determining RIP/Spark properties are NP-hard in general \cite{Weed:2017hra}. Moreover, in the general case where the dictionary $\D$ is unknown and needs to be learned from data, it is even more difficult to obtain well separated reconstructions which certainly leads to suboptimal or wrong classifiers.

Next, we provide some intuition about the limitation of the sequential approach through a toy example. Let us consider the classification of hand-written digit images belonging to two classes: ``3s'' and ``8s'' and assume that they admit $2$-sparse representations over a dictionary. Fig. \ref{fig:toy} (a-b) shows the representations of two example vectors $\x_i$ and $\x_j$ belonging to classes ``3'' and ``8'', respectively. If  only the right halves of the images are observed and no label information is provided, we are clearly faced with a problem because our observed samples from two different classes are identical, {\it i.e.} $\x^o_i = \x^o_j$. It is obvious that at least two possible $2$-sparse representations for the observed data exist as illustrated in Fig. \ref{fig:toy}(c). When the sparse solution is not unique, we may end up reconstructing wrong vectors that could not be even well separated as illustrated in  Fig. \ref{fig:toy} (d-e). In general, sequential methods using only the information of observed features are prone to fail because the non-uniqueness of solutions can make the training of a good classifier an impossible task. However, we could solve this problem by incorporating the labelling information from the very beginning as it is proposed in the following section.


\section{Simultaneous learning and coding approach}
\label{sec:gen_approach}
We propose to train the classifier and find the proper representation, not only as sparse as possible but also providing the best separation of classes. 
We want to combine the training of the classifier together with the learning of a dictionary and optimal sparse representations such that the reconstructed data vectors are compatible with observations and well separated. To do that we propose to minimize the following global cost function:
\begin{gather}\label{eq:totalcost} 
J(\Theta, \D, \s_i) =   \nonumber\\
\textstyle
\underbrace{\frac{1}{I}\sum_{i=1}^I\big\{ J_0(\Theta, \hat{\x}_i, y_i) + \lambda_{1} J_1(\D,\s_i)\big\}}_{F(\Theta, \D, \s_i)} +  
\underbrace{\frac{1}{I}\sum_{i=1}^I\big\{ \lambda_{2}J_2(\s_i)\big\}}_{G(\s_i)},
\end{gather}
with respect to $\Theta$, $\D$ and $\s_i$ ($i=1,2,\dots, I$), where $\Theta$ contains the classifier parameters, {\it i.e.} the vector of weights in a deep NN classifier architecture; $\D \in \R^{N\times P}$ ($P \ge N$) is a dictionary and $\s_i \in \Sigma_K^{P}$ are the representation coefficients such that the reconstructed data vectors are $\hat{\x}_i = \D\s_i$. 

$J_0(\Theta, \hat{\x}_i, y_i)$ is a measure of the classification error for the reconstructed sample vector $\hat{\x}_i$. Typically, we use the crossentropy measure, {\it i.e.} $J_0(\Theta, \hat{\x}_i, y_i) = -\log [p_{\Theta}(y_i | \hat{\x}_i)]$, where $p_{\Theta}(y_i | \hat{\x}_i)$ is the probability assigned by the classifier to sample $\hat{\x}_i$ as belonging to class $y_i$. 
$J_1(\D,\s_i)$ is a measure of the approximation error of the reconstruction when it is restricted to observed features, which is defined as follows: $J_1(\D,\s_i)=\frac{M_i}{N}\|\m_i \odot (\x_i - \D \s_i)\|^2$, where $\odot$ stands for the entry-wise product, $\m_i \in \R^N$ is the observation mask for sample $i$, {\it i.e.} $m_i(n) = 0$  ($1$) if data entry $\x_i(n)$ is missing (available); and  $J_2(\s_i) = \frac{1}{N}\|\s_i\|_1$ is proportional to the $\ell_1$-norm whose minimization promotes the sparsity of the representation since $\ell_1$-norm is a convenient proxy for $\ell_0$-norm \cite{Candes:2005cs}. Finally, the hyper-parameters $\lambda_1$ and $\lambda_2$ allow us to give more or less importance to the representation accuracy and its sparsity, with respect to the classification error. Intuitively, minimizing equation (\ref{eq:totalcost}) favors solutions that not only have sparse representations compatible with observed features, but also providing reconstructions that are best separated in the given classes. 

\begin{figure}[ht]
\centerline{\includegraphics[width=0.8\linewidth]{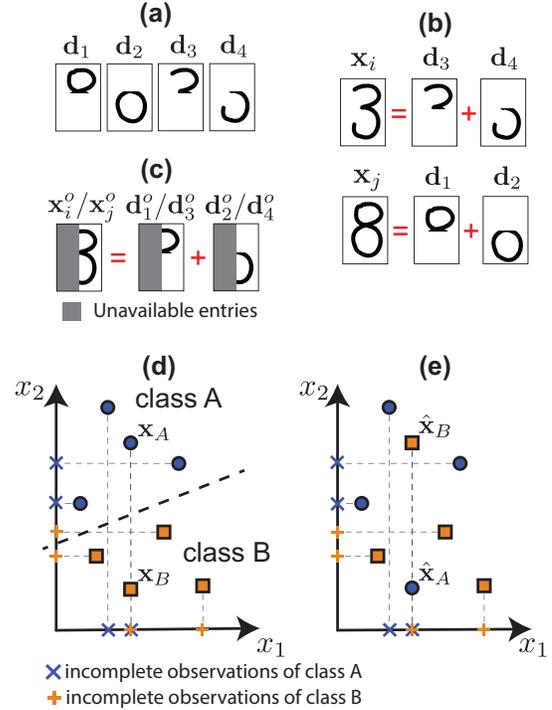}}
\caption{\footnotesize{Toy example: (a) 4 out $P$ dictionary elements $\mathbf{d}_i$ (atoms). (b) Digits ``3'' and ``8'' can be represented by combining only two atoms in the dictionary ($2$-sparse representations). (c) A left-half occluded digit ``3'' or ``8'' admits more than one $2$-sparse representation (sum of $\mathbf{d}^o_1$ and $\mathbf{d}^o_2$, or $\mathbf{d}^o_3$ and $\mathbf{d}^o_4$). (d) Linearly separable samples from two classes (A and B) having two features: $\mathbf{x} = [x_1, x_2]\in \mathds{R}^2$ where incomplete observations are taken by observing only one feature. Note that $\mathbf{x}_A$ and $\mathbf{x}_B$ belong to different classes but their observations are identical. (e) Without using label information, the sequential method could lead to wrong reconstructions of data vectors, {\it i.e.} $\hat{\mathbf{x}}_A\neq \mathbf{x}_A$ and $\hat{\mathbf{x}}_B\neq \mathbf{x}_B$ making the set of reconstructed vectors not linearly separable.}}
\label{fig:toy}
\end{figure}

\subsection{A sparsity-promoting sub-gradient optimization algorithm}\label{sec:algo}
To minimize the cost function in equation (\ref{eq:totalcost}) we propose to alternate between the optimization over $\s_i$ ($i=1,2,\dots, I$) and $\{\Theta,\D\}$ using the training dataset (incomplete). 

For fixed $\{\Theta,\D\}$, the optimization with respect to $\s_i$ is a non-smooth separable minimization sub-problem, which was extensively studied in the literature \cite{Tseng:2007gl,Hale:2008kk}. In this sub-problem, the objective function is written as the sum of $F(\Theta, \D, \s_i)$ and a non-smooth separable function $G(\s_i)$, for which highly specialized, efficient and provable convergent solvers, namely the Coordinate Gradient Descent (CGD), already exists. However, the following key differences in our setting makes it not suitable for the CGD approach: first, our function $F(\Theta, \D, \s_i)$ involves evaluation of a multi-layer NN classifier, which can be non-smooth due to involved activation functions like ReLU or others; second, and more importantly, the computation of its second derivatives (Hessian) becomes prohibitive. Therefore, we choose a simpler and standard first order (stochastic sub-gradient based) search of local minima with back-propagation. We take the strategy similar to the heuristics used in \cite{ShalevShwartz:2011vo}. To update $\s_i$, we need to subtract $\sigma_{\s} \frac{\partial J}{\partial \s_i}(j)$ from each coordinate $j$ provided that we do not cross zero in the process in order to avoid escaping from a region where  $G(\s_i)$ is differentiable. In such a case, we let the new value of $\s_i(j)$ be exactly zero. More specifically, we define $\Delta_i(j)= -\sigma_{\s} \frac{\partial J}{\partial \s_i}(j)$ and, if $\s_i(j)[\s_i(j)+\Delta_i(j)] < 0$ (zero crossing condition), we re-define $\Delta_i(j)=-\s_i(j)$; finally we update $\s_i \leftarrow \s_i + \Delta_i$. It is noted that, once a coefficient $\s_i(j)$ reaches zero at a coordinate $j$, it becomes fixed, in other words, sparsity of solution $\s_i$ is monotonically increasing with iterations.

When $\s_i$ is fixed, our problem is reduced to minimize $F(\Theta, \D, \s_i)$ with respect to $\Theta$ and $\D$, which is easily done by standard first order (stochastic gradient based) search of local minima. The algorithm proposed for the training phase is presented as Algorithm \ref{alg:train}.

In addition, for the testing phase, if the test dataset is incomplete, we need to find first the sparsest representation for the given observations, compute the reconstructions $\hat{\x}_i = \D \s_i$ and then apply the previously learned classifier to them as presented in Supp. material, Algorithm \ref{alg:test}. 

\begin{algorithm}
{\footnotesize
\caption{: \small Simultaneous classification and coding}
\label{alg:train}
\begin{algorithmic}[1]
	\small
	\REQUIRE $\{\x^o_i, y_i\}$, $i=1,2,\dots, I$, hyper-parameters $\lambda_1$ and $\lambda_2$, $N_{iter}$ and update rates $\sigma_{\Theta}$, $\sigma_{\D}$ and $\sigma_{\s}$
	\ENSURE  Weights $\Theta$ and reconstructions $\hat{\x}_i= \D \s_i, \forall i$
	\STATE \texttt{Randomly initialize} $\Theta, \D, \s_i, \forall i$
     		\FOR{$n\le N_{iter}$}
      		\STATE \texttt{Fix} $\s_i$, \texttt{update} $\Theta$ \texttt{and} $\D$:
        		\STATE $\Theta = \Theta - \sigma_{\Theta} \frac{\partial J}{\partial \Theta}$
		\STATE $\D = \D - \sigma_{\D} \frac{\partial J}{\partial \D}$
        		\STATE \texttt{Normalize columns of matrix} $\D$
        		\STATE  \texttt{Fix} $\Theta$ and $\D$, \texttt{update} $\s_i$, $\forall i$:
       	 	\STATE $\Delta_i= -\sigma_{\s} \frac{\partial J}{\partial \s_i}$, $\forall i$ 
        		\IF{$\s_i(j)[\s_i(j)+\Delta_i(j)] < 0$}
                		\STATE $\Delta_i(j)=-\s_i(j), \forall (i, j)$;
       		 \ENDIF
        		\STATE $\s_i = \s_i + \Delta_i$, $\forall i$     
      		\ENDFOR
	\STATE {\bf return} $\Theta, \D, \s_i, \hat{\x}_i = \D\s_i, \forall i$
	\end{algorithmic}	
	}
\end{algorithm}

\subsection{Theoretical analysis}
\label{sec:theory}
Here, we investigate about conditions under which a perfect classifier of the complete data can be obtained from incomplete data samples. 

\subsubsection{Logistic regression}
Let us first consider a logistic regression classifier \cite{Hastie:2009fg} where the set of parameters $\Theta = \{\w, b\}$ are a vector $\w\in \R^N$ and a scalar $b$ (bias). A perfect classifier exists if there is a hyperplane that separates both classes, {\it i.e.}, for each data vector $\x_i$: $f(\x_i) = \langle \w, \x_i \rangle + b >  0$ if $y_i=1$, and $f(\x_i)  \le 0$ if $y_i=0$. 
We consider data samples admitting a $K$-sparse representations $\x_i = \D\s_i$ with dictionary $\D\in \R^{N\times P}$ having unit-norm columns. We also assume an arbitrary pattern of missing features $\mathcal{M}_i$ such that, data samples and dictionary are partitioned as $\{\x^m_i,\x^o_i \}$ and $\{\D^m_i,\D^o_i \}$, respectively. The following lemma identifies a sufficient condition under which, if we are able to train a classifier on incomplete observations such that the reconstructed data points are well separated by a hyperplane, then the same classifier correctly separates the original (unobserved) data vectors.
\begin{lemma}[Sufficient condition type I] \label{the:main}
Suppose that we have obtained an alternative dictionary $\D' \neq \D \in \R^{N\times P}$ such that, for the incomplete observations $\x^o_i \in \R^{M_i}$, the $K$-sparse representation solutions are non-unique, {\it i.e.} $\exists \s_i, \s_i' \in \Sigma_K^P$ such that $\x^o_i = \D^o_i\s_i = \D'^o_i\s_i'$, where $\s_i \in \R^P$ are the vectors of coefficients of the true data and $\s_i'$  provides reconstructions $\hat{\x}_i = \D'\s_i'$. If a perfect classifier $\{\w, b\}$ of the reconstructions $\hat{\x}_i$ exists s.t. $|f(\hat{\x}_i)| > \epsilon_i > 0$ and
\begin{align}\label{eq:condition}
\epsilon_i > |\langle \w^m_i , \e^m_i\rangle | 
\end{align}
with $\e^m_i = \x^m_i - \hat{\x}^m_i$, then the full data vectors $\x_i$ are also perfectly separated with this classifier, in other words: $f(\x_i) = \langle \w_i, \x_{i} \rangle + b >  0$ ($\le 0$) if $y_i =1$ ($y_i=0$).
\end{lemma}

\begin{proof}
By using the missing/observed partition and omitting the sample index $i$, we can write: $f(\x) = \langle \w, \x \rangle + b = \langle \w^{o}, \x^o \rangle + \langle \w^{m}, \x^m \rangle + b$. If we add and subtract the term $\langle \w^{m}, \hat{\x}^m \rangle$ on the right left hand, arrange terms and use the fact that $\hat{\x}^o = \D'^o\s' = \x^o$, we get:
\begin{equation}
f(\x) 	 = f(\hat{\x}) + \langle \w^{m}, \e^m \rangle.
\end{equation}
Since we assumed that $f(\hat{\x}) > \epsilon >0$ (for $y_i = 1$) and $|\langle \w^{m}, \e^m \rangle| < \epsilon$, it implies that $f(\x) > 0$.
\end{proof}
Basically, condition (\ref{eq:condition}) means requiring that reconstruction vector $\hat{\x}$ has a distance $\epsilon$ to the separating hyperplane larger than  the absolute dot product between $\w$ and the residual $ \e = \x - \hat{\x}$, which of course is true when the reconstruction is accurate, {\it i.e.} $\x \approx \hat{\x}$. However, in practice, reconstructions are not accurate so we are interested in conditions under which Lemma \ref{the:main} can still holds. Below, we derive a more restrictive but useful sufficient condition:
\begin{proposition}[Sufficient condition type II]\label{the:sufcondI}
Under the same hypothesis of Lemma \ref{the:main}, the following condition is enough to guarantee a proper classifier trained on incomplete data:
\begin{equation}\label{eq:sufficient}
\epsilon > |\langle \w^{m}, \x^m\rangle | +  |\langle \w^{m}, \hat{\x}^m \rangle|.
\end{equation}
\end{proposition}
\begin{proof}
By using the fact that $|\langle \w^{m}, \e^m \rangle| = |\langle \w^{m}, \x^m - \hat{\x}^m \rangle| \le |\langle \w^{m}, \x^m \rangle| + |\langle \w^{m}, \hat{\x}^m \rangle|$, and applying Lemma \ref{the:main} the proof is completed.
\end{proof}

We highlight that, in our experiments, we were able to verify that Sufficient Condition type II is met in practice (see section \ref{sec:experiments}, Fig. \ref{fig:cond}). 

In Supp. material section \ref{sec:RIP}, we derive an additional sufficient condition based on the Restricted Isometry Property (RIP) of the dictionary $\D$ and sparsity level $K$, showing that sufficient condition (\ref{eq:sufficient}) is easier to hold for datasets admitting highly sparse representations on dictionaries as close to orthogonal ones as possible.

\subsubsection{Multilayer-perceptron}
Lemma \ref{the:main} can be straightforwardly generalized to multilayer-perceptron NNs where, if a softmax function is used at the output of the last layer then, as before, the prediction is based on the sign of the linear function:
\begin{equation}
\small
f(\x) = \langle \w, \x^{(L)}\rangle + b, \mbox{    with } \x^{(l)} = h\left( \W^T_l \x^{(l-1)} + \B_l\right),
\end{equation}
$l=1,2,\dots,L$, where $L+1$ is the total number of layers, $N_l$ is the number of neurons in layer $l$, $\w\in \R^{N_{L+1}}$ contains the weights in the last layer, $h(\cdot)$ is an activation function, {\it e.g.} ReLU, $\W_l\in \R^{N_{l-1}\times N_l}$ and $\B_l \in \R^N_l$ contain the weights and biases associated to neurons at layer $l$; and $\x^{(0)} = \x$ is the input data vector. In this case, the first layer matrix $\W_1 \in \R^{N\times N_1}$ can be partitioned into submatrices $\W^o_{1i} \in \R^{M \times N_1}$ and $\W^m_{1i} \in \R^{(N-M) \times N_1}$ according to the observed and missing input features, respectively.

\begin{proposition} []\label{the:multilayer}
Under the same conditions of Lemma \ref{the:main}, if a NN-classifier $\{{\W_l,\B_l} (l=1,2,\dots L),\w, b, h=ReLU\}$ of the reconstruction $\hat{\x}_i$ exists such that 
\begin{equation}\label{eq:condition2}
\epsilon_i  > A \max_j |\langle \W^m_{1i}(:,j)  , \e^m_i\rangle | ,
\end{equation} 
where $A = \|\w\| \prod_{l=2}^L \|\W_l \|_2$ and $\e^m_i = \x^m_i - \hat{\x}^m_i$, then the full data vector $\x_i$ is also perfectly separated, in other words: $f(\x_i)  >  0$ ($< 0$) if $y_i =1$ ($y_i=0$).
\end{proposition} 

In the proof of Lemma \ref{the:main}, we were interested in finding a bound of the output error when the input $\x$ of a classifier is perturbed, {\it i.e.} we found conditions such that $|f(\x) - f(\x + \mathbf{\delta})| < \epsilon$. By generalizing the classifier to the case of a multilayer perceptron we can derive the proof as follows:
\begin{proof}
Given a perturbation $\mathbf{\delta}^{(l-1)}\in \R^{N_l}$ at the input of layer $l-1$, {\it i.e.} $\hat{\x}^{(l-1)} = \x^{(l-1)} + \mathbf{\delta}^{(l-1)}$, it is propagated to the output of layer $l$. By writing the error at the output we obtain:
\begin{equation}
\scriptstyle
\mathbf{\delta}^{(l)} = h\left( \W_l^T \x^{(l-1)} + \B_l + \W_l^T \mathbf{\delta}^{(l-1)} \right) - h\left( \W_l^T \x^{(l-1)} + \B_l  \right),
\end{equation}
and, by using the sub-additivity of ReLU function $h(\cdot)$, {\it i.e.} $h(a+b) \le h(a) + h(b)$, we derive the following entry-wise inequality:
\begin{equation}
\small
\mathbf{\delta}^{(l)} \le h \left(  \W_l^T \mathbf{\delta}^{(l-1)} \right),
\end{equation}
and, by considering the property of ReLU activation function $\|h(\x)\| \le \|\x\|$, it turns out:
\begin{equation}\label{eq:prop1}
\small
\|\mathbf{\delta}^{(l)}\| \le \| \W_l^T \mathbf{\delta}^{(l-1)} \|,
\end{equation}

Since the last layer of the NN is a linear classifier as in the case of Lemma \ref{the:main}, we can ask that $ \langle \w, \mathbf{\delta}^{(L)}\rangle < \epsilon$. Thus, by recursively using equation (\ref{eq:prop1}), we write
\footnotesize
\begin{equation} \label{eq:inner}
\langle \w, \mathbf{\delta}^{(L)}\rangle \le \|\w\| \| \mathbf{\delta}^{(L)}\| \le \|\w\| \|\W_L\|_2 \|\W_{l-1}\|_2 \cdots \|\W_2\|_2 \|\mathbf{\delta}^{(1)} \|.
\end{equation}
\normalsize
By defining $A = \|\w\| \prod_{l=2}^L\|\W_l\|_2$, evaluating equation (\ref{eq:prop1}) with $l=1$ and taking into account that perturbation at the input of first layer is $\mathbf{\delta}^{(0)} = \e$ with $\e^o = \mathbf{0}$, we arrive at:
\begin{equation}
\small
\langle \w, \mathbf{\delta}^{(L)}\rangle \le A \| \W_1^{mT} \e^m\| \le A \max_j |\langle \W^m_{1}(:,j)  , \e^m \rangle | < \epsilon,
\end{equation}
which completes the proof.

\end{proof}

It is interesting to note that $A=1$ is attained when unit-norm filters (columns of $\mathbf{W}_l$) are orthogonal, which can be imposed by using orthogonality regularization \cite{Bansal:2018}.

\section{Experimental results}
\label{sec:experiments}
We implemented all the algorithms in Pytorch 1.0.0 on a single GPU. Implementation details are reported in Supplemental material, sections \ref{sec:implementation} and \ref{sec:hyperparameters}. The code is available at \footnote{https://github.com/ccaiafa/SimultCodClass}.

{\bf Synthetic datasets:}
We synthetically generated $I=11,000$ ($10,000$ training $+$ $1,000$ test) $K$-sparse data vectors $\x_i \in \R^{100}$ using a dictionary $\D\in \R^{100 \times 200}$ obtained from a Gaussian distribution with normalized atoms, {\it i.e.} $\|\D(:,j)\|=1, \forall j$. A random hyperplane $\{\w, b\}$ with $\w\in \R^N$, $b\in \R$ was randomly chosen dividing data vectors into two classes according to the sign of the expression $\langle \w, \x_i \rangle +b$, which defined the label $y_i$. We also controlled the degree of separation between classes by discarding all data vectors with distances to the hyperplane lower than a pre-specified threshold, {\it i.e.} $|\langle \w, \x_i \rangle +b| < d$. We used $n=10$ repetitions of each experiment with different masks and input data in order to compute statistics. 

We applied our simultaneous method ({\bf Simult.}) with hyperparameters $\lambda_1$ and $\lambda_2$ in the cost function (\ref{eq:totalcost}) tuned via cross-validation to train a logistic regression classifier on incomplete datasets with randomly distributed missing features. Then, we computed the classification accuracy on the complete test dataset and compared the results against the following standard sequential methods: \\
{\bf Sequential Sparsity based (Seq. Sp.)}: reconstructions are obtained by finding the sparsest representation compatible with the observations solving a LASSO problem. We used Algorithm \ref{alg:seq} as shown in  the Supp. material;  \\
{\bf Zero Fill (ZF)}: missing features are filled with zeros, which is equivalent to ignore unknown values;\\
{\bf Mean Unsupervised (MU)}: missing features are filled with the mean computed on the available values;\\
{\bf Mean Supervised (MS)}: as in the previous case but the mean is computed on the same class vectors only;\\
{\bf K-Nearest Neighbor (KNN)}: as in the previous case but the mean is computed on the K-nearest neighbors of the same class vectors only. 

To compare the performance of classifiers, we computed the mean accuracy $\pm$ standard error of the mean (s.e.m.), with $n=10$, on complete test datasets using all the methods for two levels of separation between classes ($d=0.0, 0.2$), two levels of sparsity ($K=4,32$) and missing features in the training dataset ranging from $25\%$ to $95\%$ as shown in Fig. \ref{fig:synth}. Our results show that the simultaneous algorithm clearly outperforms all the sequential methods. A   t-test was performed to evaluate the statistical significance with $p < 0.05$ of the difference between our algorithm and MS. It is interesting to note that, when classes has some degree of separation ($d=0.2$), using the simple MS method, can give good results but not better than our algorithm. 

\begin{figure}
 \centerline{\includegraphics[width=0.9\linewidth]{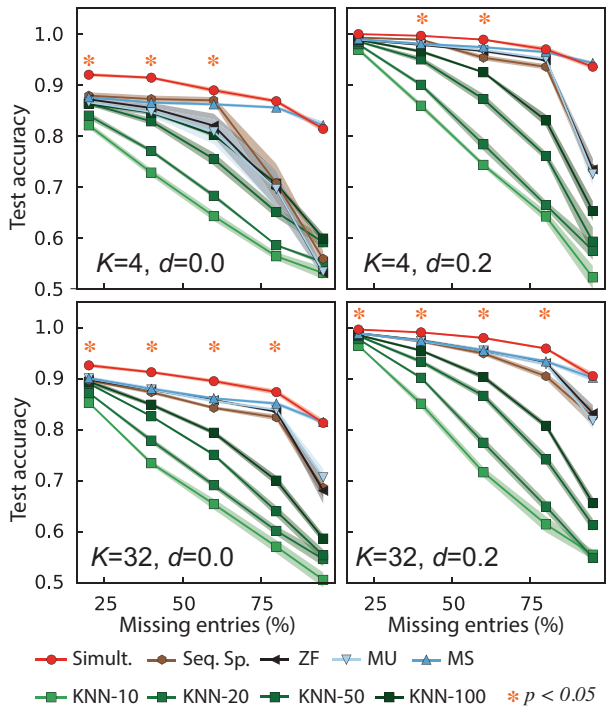}}
  \caption{\footnotesize Experimental results on synthetic dataset with random masks using our algorithm (red) and compared to various sequential methods. Test accuracy (mean $\pm$ s.e.m with $n=10$) is shown as a function of the percentage of missing features for separation of classes $d=0.0, 0.2$ and levels of sparsity $K=4,32$. Statistical significance for the difference between Simult. and MS is shown ($p < 0.05$).
}
\label{fig:synth}
\end{figure}

In the second experiment, we generated $I =10,000$ $K$-sparse data vectors $\x_i \in \R^{100}$ using $\D\in \R^{100 \times 100}$ and we evaluated the sufficient condition of equation (\ref{eq:sufficient}) on $n=10$ repetitions of the experiment with $95\%$ missing features and separation $d=0.0$. Fig. \ref{fig:cond} clearly shows that the sufficient condition is mostly met in practice, especially for highly sparse representations of input data (small $K$). This means that in practice it is not necessary to accurately reconstruct the input vectors, it is enough to capture the intrinsic characteristics of the classes such that the distances of reconstructions to the separating hyperplane satisfy the sufficient condition (\ref{fig:cond}).

\begin{figure}
 \centerline{\includegraphics[width=0.9\linewidth]{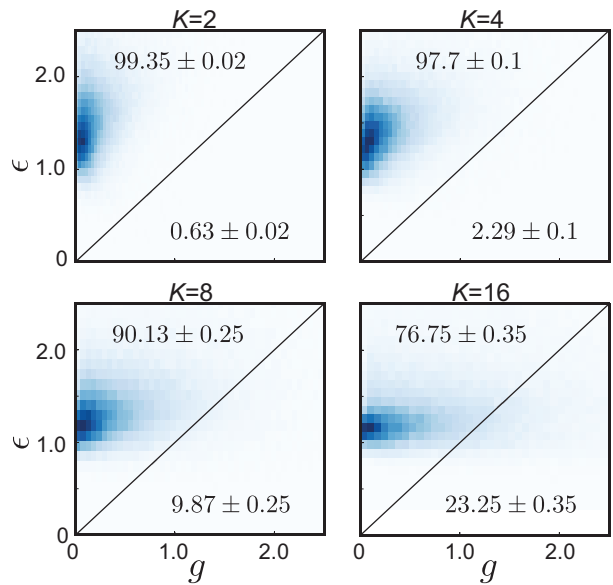}}
  \caption{\footnotesize Verification of the sufficient condition (\ref{eq:sufficient}) for various levels of sparsity $K$: 2D-histogram of $\epsilon$ versus $g = |\langle \w^{m}, \x^m\rangle | +  |\langle \w^{m}, \hat{\x}^m \rangle|$. Mean + s.e.m ($n=10$) percentage of correctly classified data samples are shown for $\epsilon > g$ and $\epsilon < g$.
}
\label{fig:cond}
\end{figure}

{\bf Benchmark datasets:}
We also considered three popular computer vision datasets: MNIST \cite{LeCun:1989vp} and Fashion \cite{2017arXiv170807747X} consisting of 70,000 images (60,000 train + 10,000 test) each; and CIFAR10 \cite{Krizhevsky:2009tr} having 60,000 images (50,000 train + 10,000 test). MNIST/Fashion datasets contains $28\times 28$ gray scale images while CIFAR10 dataset is built upon $32\times 32\times 3$ color images of different objects.
The corresponding data sample size is $N=28\times 28 = 784$ for MNIST/Fashion and $N=32\times32\times 3 = 3,072$ for CIFAR10. We considered a dictionary of size $784\times 784$ (MNIST/Fashion) and $1,024\times1,024$ (CIFAR10) and applied our simultaneous algorithm to learn the classifier on incomplete data using uniform random missing masks with several levels of missing data (25\%, 50\% and 75\%) and 50\% for random partial occlusions with MNIST/Fashion. 

We used a logistic regression classifier (single layer NN) and a 4-layer convolutional neural network \cite{LeCun:1999uh} (CNN4) for the MNIST/Fashion dataset using batch normalization (BN) \cite{2015arXiv150203167I} in the Fashion dataset. For CIFAR10 dataset, an 18-layer residual neural network, Resnet-18 \cite{He:2016ib} was implemented. We did not use any data augmentation strategy. The hyper-parameters $\lambda_1$ and $\lambda_2$ in cost function (\ref{eq:totalcost}) were adjusted by cross-validation through a grid-search, as shown in Supp. material (Table \ref{tab:crossval} and Fig. \ref{fig:hypertuning}). We compared our proposed algorithm with the following standard sequential methods: ZF, MS, KNN-10, KNN-20, KNN-50 and KNN-100; and against the recently proposed method from \cite{Smieja:2018te}, referred here as NN-GMM, which uses the same NN classifier as in our method and models missing features through GMM\footnote{\scriptsize https://github.com/lstruski/Processing-of-missing-data-by-neural-networks}. We trained the classifiers on incomplete data with random masks and tested them on complete data for MNIST and CIFAR10 datasets. The obtained mean Test Accuracy $\pm$ s.e.m ($n=10$) are reported in Table \ref{tab:compreal}. It is noted that NN-GMM provided good results with MNIST dataset compared to sequential methods, however, our simultaneous method outperformed all the methods. Interestingly, NN-GMM performed worst than any other method with CIFAR10 dataset. It seems that NN-GMM is not robust to large amount of missing data because, when we reduced the missing entries to $10\%$, the test accuracy sensibly increased to $52.57\%$. Additionally, our method showed to have little variability (small s.e.m) compared to the second best method (NN-GMM for MNIST and ZF for CIFAR10). 

\begin{table*}[ht]
\caption{\footnotesize {Test accuracy (mean $\pm$ s.e.m with $n=10$) of various methods trained on incomplete data and tested on complete ones for MNIST and CIFAR10.}}
\centering
\scalebox{0.85}{
\begin{tabular}{ |c|c|c|c|c|c|c|c|c|}
 \hline
 \multicolumn{9}{|c|}{\bf MNIST (CNN4)} \\ \hline
 \bf  Miss. &  ZF 		    	&  MS 				&  KNN10 			&  KNN20 			&  KNN50 			&  KNN100 		&  NN-GMM 			&  Simult. 			\\ \hline
 \bf  75\% &  $84.86 \pm 0.02$ 	&  $83.79 \pm 0.01$		&  $88.16 \pm 0.01$		&  $87.94 \pm 0.01$		&  $87.03 \pm 0.002$	&  $86.52 \pm 0.01$	& $ 96.36 \pm 0.12$ 		& $\bf 98.09  \pm 0.04$		\\ \hline
 \bf  50\% &  $90.13 \pm 0.06$	&  $88.55 \pm 0.01$		&  $91.36 \pm 0.02$		&  $91.11 \pm 0.02$		& $90.87 \pm 0.01$		&  $90.82 \pm 0.01$	&  $ 97.57 \pm 0.37$ 	& $\bf 98.23  \pm 0.10$		\\ \hline \hline
 \multicolumn{9}{|c|}{\bf CIFAR10 (Resnet18)} \\ \hline
  \bf  Miss. &  ZF 			&  MS 				&  KNN10 			&  KNN20 		&  KNN50 			&  KNN100 			&  NN-GMM 			&  Simult. \\ \hline
 \bf  75\% &  32.22 $\pm$ 2.09 	&  21.30 $\pm$ 0.40		&  22.84 $\pm$ 0.87		&  25.67 $\pm$ 0.80	&  26.52 $\pm$ 0.70		&  26.01 $\pm$ 0.52		& $12.10 \pm 0.61$		& $\bf 54.81 \pm 0.47 $	\\ \hline
 \bf  50\% &  46.37 $\pm$ 1.93	&  17.90 $\pm$ 0.94		&  30.94 $\pm$ 0.54		&  29.68 $\pm$ 0.46	& 30.01 $\pm$ 0.51		&  26.23 $\pm$ 1.01		&  $14.02 \pm 0.75$		& $\bf 62.50 \pm  0.95$	\\ \hline
\end{tabular}
}
\label{tab:compreal}
\end{table*}

In Table \ref{tab:AccuTest}, test accuracies obtained when the learned model is applied to incomplete and complete test data, are shown. The right-most column shows the baseline results obtained by training the model on complete datasets using a CNN4 \cite{LeCun:1999uh} and a  Resnet-18 \cite{He:2016ib}, whose implementations can be found at \footnote{\scriptsize  https://github.com/pytorch/examples/tree/master/mnist} and \footnote{\scriptsize https://github.com/kuangliu/pytorch-cifar}. It is interesting to note that for the logistic regression classifier, we obtained better results when training with incomplete data rather than using complete data. Also, it is highlighted that training on incomplete data with 50\% or fewer random missing features, provides similar test accuracy as training on complete data for MNIST dataset. This could be explained by noting two facts: (1) random missing features is similar to applying dropout, with the exception that missing data do not change during training; and (2) our model has more parameters (Dictionary + sparse coefficients + Linear layer) compared to the baseline logistic regression classifier. To provide a deeper understanding of this effect, we ran the baseline with Dropout at the input and we obtained: 91.95\%, 91.97\% and 92.01\% for $p=0.0$, $0.1$ and $0.25$, respectively, which shows that the improvement we obtained with our method is not solely caused by a dropout alike behavior. 

\begin{table*}[ht]
\centering
\caption{\footnotesize{Test Accuracies obtained with our method on MNIST, Fashion and CIFAR10 datasets training with incomplete data and testing on incomplete/complete data. Baseline results obtained by training the models on complete data are shown for reference in the right-most column.}}
\scalebox{0.85}{
\begin{tabular}{ |c|c|c|c|c|c|c|c|c|c|c|  }
 \hline
 \multirow{3}{3em}{\bf Dataset} &  \multirow{3}{4em}{\bf Classifier} & \multicolumn{6}{|c|}{\bf Random missing features} & \multicolumn{2}{|c|}{\bf Occlusion} & \bf Baseline \\
  &   & \multicolumn{6}{|c|}{\bf \%Train / \%Test} & \multicolumn{2}{|c|}{\bf \%Train / \%Test} &\bf \%Train / \%Test \\ \cline{3-11}
 & 									& \bf 75/75 	& \bf 50/50 	&\bf  25/25 	&\bf  75/0 		&\bf  50/0 		&\bf  25/0 		&\bf  50/50 	&\bf  50/0		&\bf  0/0 \\
 \hline
 \multirow{2}{3em}{\bf MNIST}  & \bf Log. Reg. 	& $90.45$ 	& $93.68$ & 	$94.14$ & 	$91.94$ 		& $93.44$ 	& $94.43$ 	& - 			& - 			& $91.95$ \\ \cline{2-11}
 &\bf  CNN4 							& $94.62$ 	& $98.34$ 	& $98.94$ 	& $98.09$ 	& $98.23$ 	& $98.95$ 	& $88.55$ 	& $91.37$ 	& $98.95$ \\
 \hline
 \bf Fashion & \bf CNN4+BN 				& $83.71$ 	& $86.09$ 	& $86.38$ 	& $86.39$ 	& $87.11$ 	& $87.04$ 	& $81.73$ 	& $82.47$ 	& $90.76$ \\
 \hline
 \bf CIFAR10 & \bf Resnet18 				& $53.82$ & 	$61.08$ 		& $63.73$ 	& $54.81$ 	& $62.50$ 	& $63.87$ 	& - 			& - 			& $80.13$ \\
 \hline
\end{tabular}
}
\label{tab:AccuTest}
\end{table*}

In Fig. \ref{fig:examples}, we present some randomly selected visual examples comparing the original images in the MNIST/Fashion test dataset, their observations using random masks and partial occlusions, and the reconstructions using the dictionary learned from the incomplete training data. It is clear that, despite the reconstructions may be not very similar to the original images (see ``5'' digit example), they clearly own the properties of the class to which they belong to. Additional examples are provided in Supp. material (Figs. \ref{fig:moreexamplesMNIST} and \ref{fig:moreexamplesFashion}).

\begin{figure}
 \centerline{\includegraphics[width=1.0\linewidth]{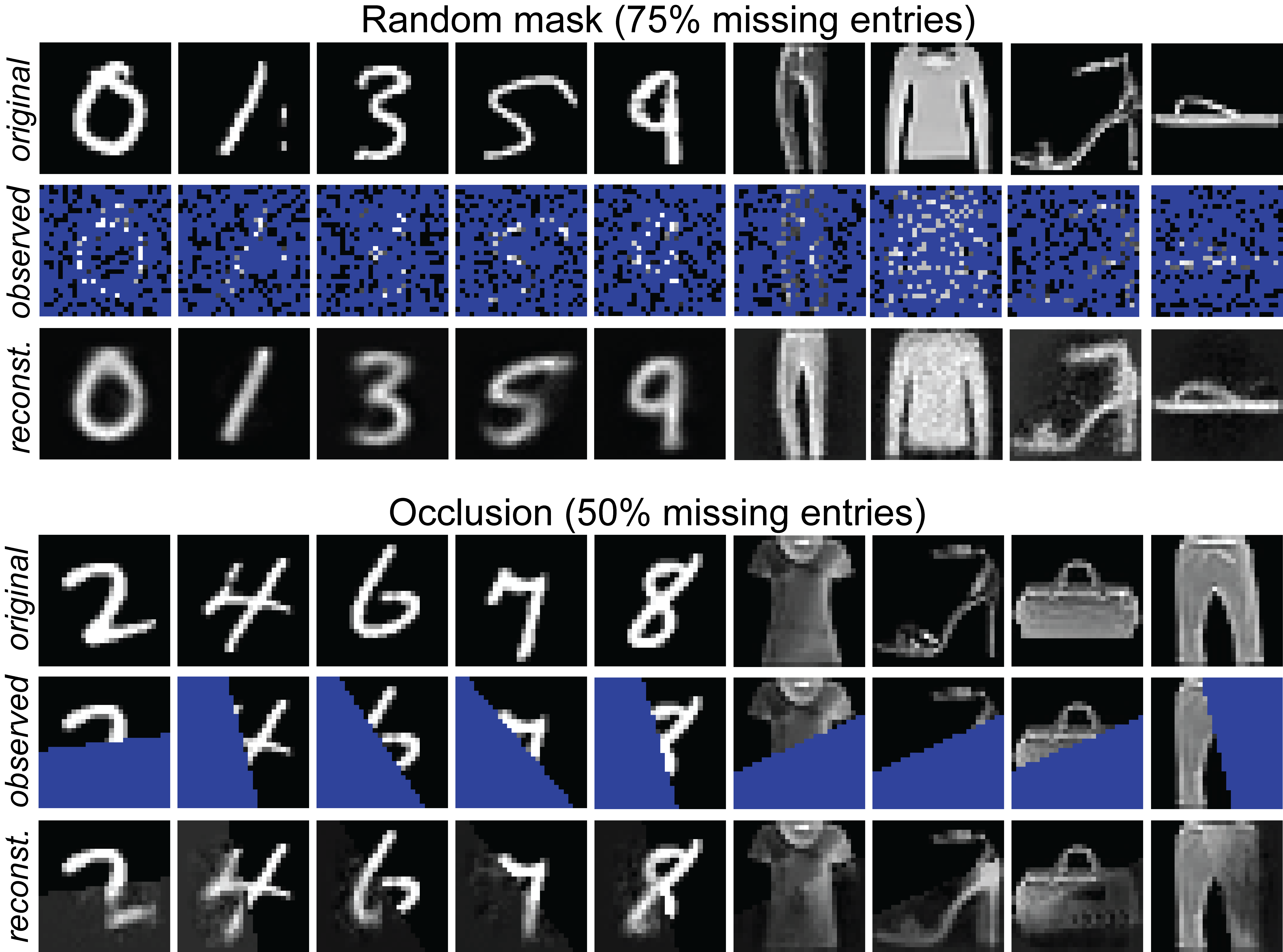}}
  \caption{\footnotesize Original (top), observed (middle) and reconstructed (bottom) MNIST and Fashion test images.
}
\label{fig:examples}
\end{figure}

\section{Discussion}
\label{sec:conclusions}
It is well known that sparse coding has the ability to accurately model complex distributions of data, such as natural signals (images, audio, EEG, etc). In this work, we demonstrated that assuming a sparse representation for input data allows for the successful training of a general NN when incomplete data is given outperforming traditional sequential approaches and other start-of-the-art methods. It is highlighted that our method can be used with potentially any deep NN architecture, thus relying on their extraordinary capability to accommodate complex decision boundaries as usually needed in modern machine learning.  

Our method overcomes well known issues of previous approaches: (1) compared to imputation methods, our algorithm successfully incorporates the labelling information into the modeling of missing features; (2) sparse coding allows for a simple way to train dictionaries through linear methods such as stochastic gradient descent with back-propagation compared to the very expensive EM estimators for GMM used in probabilistic generative models, or SVD based algorithms for matrix rank minimization in matrix completion; (3) sparse coding can be more accurate modeling missing values in natural signals compared to GMM, especially for high dimensional data where GMM may require a huge number of parameters making it computationally prohibitive.

We analyzed the limitations of the classical imputation approach and demonstrated through experiments with synthetical and real-world datasets that our simultaneous algorithm always outperforms them for various cases such as LASSO, zero-filling, supervised/unsupervised mean and KNN based methods as well as the {\it state-of-the-art} method based on NNs and GGM recently proposed in \cite{Smieja:2018te}. Nevertheless, our experimental results on synthetic and real-world dataset showed that, even though we only constrained dictionaries to have unit-norm columns but not enforcing any other kind of constraint like maximum coherence, the obtained results seem to be satisfactory enough. However, further analysis on the required properties of dictionaries could provide deeper insights and alternative ways to improve the algorithm, which we aim to address in a future work.

While current simple sub-gradient based optimization approach provided satisfactory results in terms of performance, it is remarked that observed convergence is slow requiring a thousand of iterations sometimes. We believe, it could be improved by trying to incorporate some second-order derivatives information for computing the updates. Although, full Hessian computation becomes prohibitive with multi-layer NNs a diagonal approximation approach could be explored. Also, a rigorous convergence analysis in the line of the analysis in \cite{Tseng:2007gl,Hale:2008kk} and taking special properties of multi-layer NN classifier functions can be conducted in a future work.

Finally, we provided theoretical insights of the problem by providing sufficient conditions under which, if it is possible to train a classifier on incomplete observations so that its reconstructions are well separated by a hyperplane, then the same classifier also correctly separates the original (unobserved) data samples. 

{\bf Acknowledgments:} We are thankful for the RAIDEN computing system and its support team at RIKEN AIP, Tokyo. This work was supported by the JSPS KAKENHI (Grant No. 20H04249, 20H04208).


\begin{thebibliography}{10}
\providecommand{\url}[1]{#1}
\csname url@samestyle\endcsname
\providecommand{\newblock}{\relax}
\providecommand{\bibinfo}[2]{#2}
\providecommand{\BIBentrySTDinterwordspacing}{\spaceskip=0pt\relax}
\providecommand{\BIBentryALTinterwordstretchfactor}{4}
\providecommand{\BIBentryALTinterwordspacing}{\spaceskip=\fontdimen2\font plus
\BIBentryALTinterwordstretchfactor\fontdimen3\font minus
  \fontdimen4\font\relax}
\providecommand{\BIBforeignlanguage}[2]{{%
\expandafter\ifx\csname l@#1\endcsname\relax
\typeout{** WARNING: IEEEtran.bst: No hyphenation pattern has been}%
\typeout{** loaded for the language `#1'. Using the pattern for}%
\typeout{** the default language instead.}%
\else
\language=\csname l@#1\endcsname
\fi
#2}}
\providecommand{\BIBdecl}{\relax}
\BIBdecl

\bibitem{Chechik:2008vk}
G.~Chechik, G.~Heitz, G.~Elidan, P.~Abbeel, and D.~Koller, ``{Max-margin
  Classification of Data with Absent Features.}'' \emph{Journal of Machine
  Learning Research (JMLR)}, vol.~9, pp. 1--21, 2008.

\bibitem{RUBIN:2014vw}
R.~J.~A. Little and D.~B. Rubin, \emph{{Statistical Analysis with Missing
  Data}}.\hskip 1em plus 0.5em minus 0.4em\relax John Wiley {\&} Sons, Aug.
  2014.

\bibitem{2019arXiv190206931J}
J.~Josse, N.~Prost, E.~Scornet, and G.~Varoquaux, ``{On the consistency of
  supervised learning with missing values},'' \emph{arXiv.org}, p.
  arXiv:1902.06931, Feb. 2019.

\bibitem{GarciaLaencina:2009in}
P.~J. Garc{\'\i}a-Laencina, J.-L. Sancho-G{\'o}mez, and A.~R. Figueiras-Vidal,
  ``{Pattern classification with missing data: a review},'' \emph{Neural
  Computing and Applications}, vol.~19, no.~2, pp. 263--282, 2009.

\bibitem{Ghahramani:vv}
Z.~Ghahramani and M.~I. Jordan, ``{Supervised learning from incomplete data via
  an EM approach.}'' in \emph{NIPS}, 1993.

\bibitem{Williams:2005kq}
D.~Williams, X.~Liao, Y.~Xue, and L.~Carin, ``{Incomplete-data classification
  using logistic regression.}'' \emph{ICML}, 2005.

\bibitem{Bhattacharyya:2004ta}
C.~Bhattacharyya, P.~K. Shivaswamy, and A.~J. Smola, ``{A Second Order Cone
  programming Formulation for Classifying Missing Data.}'' in \emph{NIPS},
  2004.

\bibitem{Liao:2007fd}
X.~Liao, H.~Li, and L.~Carin, ``{Quadratically gated mixture of experts for
  incomplete data classification.}'' in \emph{ICML}, 2007.

\bibitem{Dick:2008fj}
U.~Dick, P.~Haider, and T.~Scheffer, ``{Learning from incomplete data with
  infinite imputations.}'' in \emph{ICML}, 2008.

\bibitem{Lin:2006bu}
T.~I. Lin, J.~C. Lee, and H.~J. Ho, ``{On fast supervised learning for normal
  mixture models with missing information},'' \emph{Pattern Recognition},
  vol.~39, no.~6, pp. 1177--1187, 2006.

\bibitem{2012arXiv1209.0521D}
O.~Delalleau, A.~C. Courville, and Y.~Bengio, ``{Efficient EM Training of
  Gaussian Mixtures with Missing Data},'' \emph{arXiv}, vol. cs.LG, p.
  1209.0521, 2012.

\bibitem{Goldberg:2010wg}
A.~B. Goldberg, X.~Zhu, B.~Recht, J.-M. Xu, and R.~D. Nowak, ``{Transduction
  with Matrix Completion - Three Birds with One Stone.}'' in \emph{NIPS}, 2010.

\bibitem{Hazan:2015tt}
E.~Hazan, R.~Livni, and Y.~Mansour, ``{Classification with Low Rank and Missing
  Data.}'' in \emph{ICML}, 2015.

\bibitem{2018arXiv180205380H}
S.-J. Huang, M.~Xu, M.-K. Xie, M.~Sugiyama, G.~Niu, and S.~Chen, ``{Active
  Feature Acquisition with Supervised Matrix Completion},'' \emph{arXiv}, vol.
  cs.LG, p. 1802.05380, 2018.

\bibitem{Smieja:2018te}
M.~Smieja, L.~Struski, J.~Tabor, B.~Zielinski, and P.~Spurek, ``{Processing of
  missing data by neural networks.}'' in \emph{NeurIPS}, 2018.

\bibitem{Lee:2007wn}
H.~Lee, A.~Battle, R.~Raina, and A.~Y. Ng, ``{Efficient sparse coding
  algorithms.}'' in \emph{NIPS}, 2006.

\bibitem{Mairal:2009ku}
J.~Mairal, F.~R. Bach, J.~Ponce, and G.~Sapiro, ``{Online dictionary learning
  for sparse coding.}'' in \emph{ICML}, 2009.

\bibitem{LeCun:2010tj}
K.~Gregor and Y.~LeCun, ``{Learning Fast Approximations of Sparse Coding.}'' in
  \emph{ICML}, 2010.

\bibitem{Mallat:2009wr}
S.~Mallat, \emph{{A Wavelet Tour of Signal Processing}}, ser. The Sparse
  Way.\hskip 1em plus 0.5em minus 0.4em\relax Academic Press, 2009.

\bibitem{Eldar:2012wf}
M.~Davenport, M.~F. Duarte, Y.~C. Eldar, and G.~Kutyniok, \emph{{Introduction
  to compressed sensing}}, ser. Theory and Applications.\hskip 1em plus 0.5em
  minus 0.4em\relax Cambridge: Cambridge University Press, 2012.

\bibitem{1996Natur.381..607O}
B.~A. Olshausen and D.~J. Field, ``{Emergence of simple-cell receptive field
  properties by learning a sparse code for natural images.}'' \emph{Nature},
  vol. 381, no. 6583, pp. 607--609, 1996.

\bibitem{Olshausen:1997p405}
B.~Olshausen and D.~Field, ``{Sparse coding with an overcomplete basis set: A
  strategy employed by V1?}'' \emph{Vision research}, vol.~37, no.~23, pp.
  3311--3325, 1997.

\bibitem{Huang:2006ud}
K.~Huang and S.~Aviyente, ``{Sparse Representation for Signal
  Classification.}'' in \emph{NIPS}, 2006.

\bibitem{MairalNIPS2008}
J.~Mairal, J.~Ponce, G.~Sapiro, A.~Z. A.~i. neural, and {2009}, ``{Supervised
  dictionary learning},'' in \emph{NIPS}, 2008.

\bibitem{Ramirez:2010jm}
I.~Ramirez, P.~Sprechmann, and G.~Sapiro, ``{Classification and clustering via
  dictionary learning with structured incoherence and shared features.}'' in
  \emph{CVPR}, 2010.

\bibitem{Sprechmann:2010jt}
P.~Sprechmann and G.~Sapiro, ``{Dictionary learning and sparse coding for
  unsupervised clustering.}'' in \emph{ICASSP}, 2010.

\bibitem{2011ISPM...28...27T}
I.~Tosic and P.~Frossard, ``{Dictionary Learning},'' \emph{IEEE Signal
  Processing Magazine}, vol.~28, no.~2, pp. 27--38, 2011.

\bibitem{Jenatton:2010ua}
R.~Jenatton, J.~Mairal, G.~Obozinski, and F.~R. Bach, ``{Proximal Methods for
  Sparse Hierarchical Dictionary Learning.}'' in \emph{ICML}, 2010.

\bibitem{Weed:2017hra}
J.~Weed, ``{Approximately Certifying the Restricted Isometry Property is
  Hard},'' \emph{IEEE Transactions on Information Theory}, vol.~64, no.~8, pp.
  5488--5497, 2017.

\bibitem{Candes:2005cs}
E.~J. Cand{\`e}s and T.~Tao, ``{Decoding by Linear Programming},'' \emph{IEEE
  Transactions on Information Theory}, vol.~51, no.~12, pp. 4203--4215, 2005.

\bibitem{Tseng:2007gl}
P.~Tseng and S.~Yun, ``{A coordinate gradient descent method for nonsmooth
  separable minimization},'' \emph{Mathematical Programming}, vol. 117, no.
  1-2, pp. 387--423, 2007.

\bibitem{Hale:2008kk}
E.~T. Hale, W.~Yin, and Y.~Zhang, ``{Fixed-Point Continuation for
  $\ell_1$-Minimization: Methodology and Convergence},'' \emph{SIAM J. OPTIM.},
  vol.~19, no.~3, pp. 1107--1130, 2008.

\bibitem{ShalevShwartz:2011vo}
S.~Shalev-Shwartz and A.~Tewari, ``{Stochastic Methods for l1-regularized Loss
  Minimization.}'' \emph{Journal of Machine Learning Research (JMLR)}, 2011.

\bibitem{Hastie:2009fg}
T.~Hastie, R.~Tibshirani, and J.~H. Friedman, \emph{{The elements of
  statistical learning - data mining, inference, and prediction, 2nd
  Edition.}}, ser. Springer Series in Statistics.\hskip 1em plus 0.5em minus
  0.4em\relax New York, NY: Springer New York, 2009.

\bibitem{Bansal:2018}
N.~Bansal, X.~Chen, and Z.~Wang, ``{Can We Gain More from Orthogonality
  Regularizations in Training Deep CNNs?}'' in \emph{NeurIPS}, 2018.

\bibitem{LeCun:1989vp}
Y.~LeCun, B.~E. Boser, J.~S. Denker, D.~Henderson, R.~E. Howard, W.~E. Hubbard,
  and L.~D. Jackel, ``{Handwritten Digit Recognition with a Back-Propagation
  Network.}'' in \emph{NIPS}, 1989.

\bibitem{2017arXiv170807747X}
H.~Xiao, K.~Rasul, and R.~Vollgraf, ``{Fashion-MNIST: a Novel Image Dataset for
  Benchmarking Machine Learning Algorithms},'' \emph{arXiv}, vol. cs.LG, p.
  arXiv:1708.07747, 2017.

\bibitem{Krizhevsky:2009tr}
A.~Krizhevsky, ``{Learning multiple layers of features from tiny images},''
  Ph.D. dissertation, Toronto University, Toronto, 2009.

\bibitem{LeCun:1999uh}
Y.~LeCun, P.~Haffner, L.~Bottou, and Y.~Bengio, ``{Object Recognition with
  Gradient-Based Learning.}'' in \emph{Shape, Contour and Grouping in Computer
  Vision}, 1999.

\bibitem{2015arXiv150203167I}
S.~Ioffe and C.~Szegedy, ``{Batch Normalization - Accelerating Deep Network
  Training by Reducing Internal Covariate Shift.}'' in \emph{ICML}, 2015.

\bibitem{He:2016ib}
K.~He, X.~Zhang, S.~Ren, and J.~Sun, ``{Deep Residual Learning for Image
  Recognition},'' in \emph{2016 IEEE Conference on Computer Vision and Pattern
  Recognition (CVPR}.\hskip 1em plus 0.5em minus 0.4em\relax IEEE, 2016, pp.
  770--778.

\bibitem{Naumova:2017hv}
V.~Naumova and K.~Schnass, ``{Dictionary learning from incomplete data for
  efficient image restoration.}'' \emph{EUSIPCO}, 2017.

\bibitem{Mairal:fi}
J.~Mairal, M.~Elad, and G.~Sapiro, ``{Sparse Representation for Color Image
  Restoration},'' \emph{IEEE Transactions on Image Processing}, vol.~17, no.~1,
  pp. 53--69, Jan. 2008.

\end{thebibliography}


\section{Supplemental material}
\subsection{Additional pseudocodes}
\label{sec:alg}
Here, additional pseudocode of the algorithms discussed in the paper are provided. Once the classifier is trained by using Algorithm \ref{alg:train}, we are able to apply it to incomplete test data by using Algorithm \ref{alg:test}, where for fixed $\Theta$ and $\D$, we need to find the corresponding sparse coefficients $\s_i$, compute the full data vector estimations and, finally, apply the classifier.

A sparsity-based sequential method is presented in Algorithm \ref{alg:seq} (sequential approach), which consists on learning first the optimal dictionary $\D$ and sparse coefficients $\s_i$ compatible with the incomplete observations (dictionary learning and coding phase), followed by the training phase, where the classifier weights are tuned in order to minimize the classification error of the reconstructed input data vectors $\hat{\x}_i = \D\s_i$. It is noted that for the imputation stage (lines 2-12) other and more specialized dictionary learning algorithms with missing data can be applied, such as the ones proposed in \cite{Naumova:2017hv} for high-dimensional data or \cite{Mairal:fi} for color image data.

\begin{algorithm}
{\footnotesize
\caption{: Testing on incomplete data}
\label{alg:test}
\begin{algorithmic}[1]
\REQUIRE Incomplete data vectors $\{\x^o_i\}$, $i=1,2,\dots, I$, classifier parameters $\Theta$, dictionary $\D$, hyper-parameters $\lambda_1$ and $\lambda_2$, number of iterations $N_{iter}$ and update rate $\sigma_{\s}$ 
\ENSURE   ${\hat{y}_i}$ and reconstructions $\hat{\x}_i= \D \s_i, \forall i$
\STATE {\bf Sparse coding stage:} for fixed dictionary $\D$ find sparse representations of observations $\x^o_i$
\STATE \texttt{Initialize $\s_i, \forall i$ randomly}
      \FOR{$n\le N_{iter}$}
        \STATE $\Delta_i= -\sigma_{\s} \big[\lambda_1 \frac{\partial J_1}{\partial \s_i} +  \lambda_2 \frac{\partial J_2}{\partial \s_i}\big]$, $\forall i$ 
        \IF{$\s_i(j)[\s_i(j)+\Delta_i(j)] < 0$}
                \STATE $\Delta_i(j)=-\s_i(j)$; avoid zero crossing
        \ENDIF
        \STATE $\s_i = \s_i + \Delta_i$, $\forall i$     
      \ENDFOR
      \STATE $ \hat{\x}_i = \D\s_i, \forall i$; Compute reconstructions
      \STATE {\bf Classification stage:} apply classifier to reconstructions $\hat{\x}_i$
      \STATE $\hat{y}_i = \argmax_y (p^{y}_{\Theta}(\hat{\x}_i))$
\STATE {\bf return}  $\Theta, \hat{y}_i, \s_i, \hat{\x}_i, \forall i$
\end{algorithmic}
}
\end{algorithm}

\begin{algorithm}
{\footnotesize
\caption{: Sequential sparsity based approach}
\label{alg:seq}
\begin{algorithmic}[1]
\REQUIRE Incomplete data vectors and their labels $\{\x ^o_i, y_i\}$, $i=1,2,\dots, I$, hyper-parameters $\lambda_1$ and $\lambda_2$, number of iterations $N_{iter}$ and update rate $\sigma_{\Theta}$, $\sigma_{\D}$ and $\sigma_{\s}$ 
\ENSURE  Classifier weights $\Theta$ and reconstructions $\hat{\x}_i= \D \s_i, \forall i$
\STATE \texttt{Randomly initialize $\D, \s_i, \forall i$}
\STATE \textbf{Imputation stage: learning of $\D$ and $\s_i$}
      \FOR{$n\le N_{iter}$}
        \STATE $\D = \D - \sigma_{\D} \frac{\partial J_1}{\partial \D}$
        \STATE Normalize columns of matrix $\D$
        \STATE $\Delta_i= -\sigma_{\s} \big[\lambda_1 \frac{\partial J_1}{\partial \s_i} +  \lambda_2 \frac{\partial J_2}{\partial \s_i}\big]$, $\forall i$ 
        \IF{$\s_i(j)[\s_i(j)+\Delta_i(j)] < 0$}
                \STATE $\Delta_i(j)=-\s_i(j)$; avoid zero crossing
        \ENDIF
        \STATE $\s_i = \s_i + \Delta_i$, $\forall i$     
      \ENDFOR
\STATE $ \hat{\x}_i = \D\s_i, \forall i$; Compute reconstructions
\STATE \textbf{Training stage: update $\Theta$}
	 \FOR{$n\le N_{}$}
          \STATE $\Theta = \Theta - \sigma_{\Theta} \frac{\partial J_0}{\partial \Theta}$;
	 \ENDFOR
\STATE {\bf return} $\Theta, \D, \s_i, \hat{\x}_i = \D\s_i, \forall i$
\end{algorithmic}

}
\end{algorithm}


\subsection{A condition based on RIP and sparsity} \label{sec:RIP}
{\bf The Restricted Isometry Property (RIP):}  An overcomplete dictionary $\D$ satisfies the RIP of order $K$ if there exists $\delta_K\in [0,1)$ s.t.
\begin{equation}\label{eq:RIP}
(1-\delta_K) \|\s\|^2_2 \le \|\D\s\|^2_2 \le (1+\delta_K) \|\s\|^2_2,
\end{equation}
holds for all  $\s \in \Sigma_K^P$. RIP was introduced in \cite{Candes:2005cs} and characterizes matrices which are nearly orthonormal when operating on sparse vectors. 

In the following theorem, we show that, by imposing conditions on the sparsity level of the representation and the RIP constant of a sub-matrix of the dictionary, we can guarantee to meet the sufficient condition (\ref{eq:sufficient}). 

\begin{theorem} \label{the:RIPbased}
Given a dataset $\{\x_i, y_i\}$, $i =1,2,\dots, I$ with normalized data vectors ($\|\x_i\|\le 1$) admitting a $K$-sparse representation over a dictionary $\D\in \R^{N\times P}$ with unit-norm columns, whose sub-matrices $\D^m_i$ satisfy the RIP of order $K$ with constant $\delta^i_K$, and suppose that, we have obtained an alternative dictionary $\D'\in \R^{N\times P}$, whose sub-matrices $\D'^m_i$ also satisfy the RIP of order $K$ with constant $\delta^i_K$ such that, for the incomplete observation $\x^o_i \in \R^{M_i}$, the $K$-sparse representation solution is non-unique, i.e. $\exists \s_i, \s_i' \in \Sigma_K^P$ such that $\x^o_i = \D^o_i\s_i = \D'^o_i\s_i'$, where $\s_i \in \R^P$ is the vector of coefficients of the true data, i.e. $\x_i = \D \s_i$ and $\s_i'$  provides a plausible reconstruction through $\hat{\x}_i = \D'\s_i'$ with $\|\hat{\x}_i\| \le 1$. If a perfect classifier $\{\w, b\}$ of the reconstruction $\hat{\x}_i$ exists such that $|f(\hat{\x})| = |\langle \w, \hat{\x} \rangle + b |> \epsilon_i >0$ and
\begin{equation}\label{eq:conditionRIP}
\epsilon_i > 2 \|\w^m_i\|_1\sqrt{\frac{K}{1-\delta^i_K}},
\end{equation}
then the full data vector $\x_i$ is also perfectly separated with this classifier, in other words: $f(\x_i) = \langle \w_i, \x_{i} \rangle + b >  0$ ($\le 0$) if $y_i =1$ ($y_i=0$).
\end{theorem}

\begin{proof}
Let us prove that the sufficient condition (\ref{eq:sufficient}) is met under the hypothesis of Theorem \ref{the:RIPbased}.
Taking into account that $\x_i^m = \D_i^m \s_i$, we can write
\begin{align}\label{eq:abs}
\small
| \langle \w^{m}, \D_i^m\s_i \rangle |  & = \Big|\sum_{j=1}^{N-M_i} \w^m(j) \sum_{n=1}^N\D_i^m(j,n) \s_i(n) \Big| \nonumber \\
			 & \le   \sum_{j=1}^{N-M_i} |\w^m(j)| \sum_{n=1}^N |\D_i^m(j,n)| |\s_i(n)|.
\end{align}
Since we assumed normalized vectors $\|\x_i\| \le 1$, by applying the left-hand side of the RIP we obtain: $\|\s_i\| \le 1/\sqrt{1-\delta^i_K}$, and, taking into account that $\|\s_i\|_1 \le \sqrt{K}\|\s_i\|$ and $ |\D_i^m(j,n)| \le 1$ (columns of $\D$ are unit-norm), we obtain:
\begin{equation}\label{eq:eq1}
\small
| \langle \w^{m}, \D_i^m\s_i \rangle |  \le  \sqrt{\frac{K}{1-\delta^i_K}} \sum_{j=1}^{N-M_i} |\w^m(j)| = \sqrt{\frac{K}{1-\delta^i_K}}\|\w^m_i\|_1.
\end{equation}
Similarly, using $\hat{\x_i}^m = \D_i'^{m} \s_i'$, we can obtain that 
\begin{equation}\label{eq:eq2}
| \langle \w^{m}, \D_i'^m\s_i' \rangle | \le  \sqrt{\frac{K}{1-\delta^i_K}}\|\w^m_i\|_1.
\end{equation}
Putting equations (\ref{eq:eq1}) and (\ref{eq:eq2}) together with equation (\ref{eq:conditionRIP}) complete the proof of the sufficient condition (\ref{eq:sufficient}).
\end{proof}


\subsection{Experimental results details}
\subsubsection{Implementation}\label{sec:implementation}
We implemented all the algorithms in Pytorch 1.0.0 on a single GPU. The code is available at\footnote{ https://github.com/ccaiafa/SimultCodClass.}. 

Initializations of dictionary $\D$ and coefficients $\s_i$ were made at random. However, we think some improvements in convergence could be achieved by using some dedicated dictionaries such as the case of Wavelet or Cosine Transform matrices. 

To update NN weights ($\Theta$), we used standard Stochastic Gradient Descent (SGD) with learning rate $\sigma_{\Theta}$ and momentum $m$, while for updating dictionary $\D$ and vector coefficients $\s_i$, we used fixed update rate $\sigma = \sigma_{\D} = \sigma_{\s}$. It is noted that we used different update rates for training and testing stages. In Table \ref{tab:settings}, we report the settings used for experiments for MNIST and CIFAR10 datasets, which includes Number of iterations $N_{iter}$, batch size $bs$, learning rate $\sigma_{\Theta}$, momentum $m$, update rate $\sigma$ (training and test).

\subsubsection{Hyperparameter tuning}\label{sec:hyperparameters}
In Table \ref{tab:crossval} we present the results of the grid search for hyper-parameter tuning on MNIST and CIFAR10 datasets. We fit our model to the training dataset for a range of values of parameters $\lambda_1$ and $\lambda_2$ and apply it to a validation data set. Figure \ref{fig:hypertuning} shows the validation accuracy obtained with different classifiers and levels of missing entries for MNIST dataset. 

\begin{table}[]
\caption{\footnotesize{Experimental settings for MNIST and CIFAR10 datasets: Number of iterations $N_{iter}$, batch size $bs$, learning rate $\sigma_{\Theta}$, momentum $m$, update rate $\sigma$ (training and test) }}
\centering
\scalebox{0.73}{
\begin{tabular}{ |c|c|c|c|c|c|c|c|  }
 \hline
 {\bf Dataset} 					&  {\bf Classifier} 	&  $N_{iter}$ 	& $bs$	 	& $\sigma_{\Theta}$ 	& $m$ 	& $\sigma$ (train) 	& $\sigma$ (test) \\ 
 \hline
 \multirow{2}{3em}{\bf MNIST}  		& \bf Log. Reg. 		& 3000 		& 250 		& 0.1 	& 0.5 	& 1.0  			& 5.0    \\ \cline{2-8}
 							&\bf  CNN4 		& 3500 		& 250 		& 05	 	& 0.5 	& 0.4 			& 0.5  \\
 \hline
 \bf CIFAR10 					& \bf Resnet18 		& 1000 		& 64	 		& 0.01 	& 0.5 	& 1.0 			& 2.5   \\
 \hline
\end{tabular}
}
\label{tab:settings}
\end{table}

\subsubsection{Additional visual results}
To visually evaluate our results, additional randomly selected examples of original (complete) images of the test dataset in MNIST and Fashion, together with their given incomplete observations and obtained reconstructions, are shown in Figure \ref{fig:moreexamplesMNIST} and Figure \ref{fig:moreexamplesFashion} 

\begin{table}[]
\caption{\footnotesize{Hyper-parameter tuning: crossvalidated hyperparameters $\lambda_1$ and $\lambda_2$ obtained for MNIST and CIFAR10 datasets with the classifiers used in our experiments.}}
\centering
\scalebox{0.63}{
\begin{tabular}{ |c|c|c|c|c|c|c|c|c|c|  }
 \hline
 \multirow{3}{3em}{\bf Dataset} &  \multirow{3}{4em}{\bf Classifier} & \multicolumn{6}{|c|}{\bf Random missing entries} & \multicolumn{2}{|c|}{\bf Occlusion} \\ \cline{3-10}
 & &  \multicolumn{2}{|c|}{\bf 75\%} & \multicolumn{2}{|c|}{\bf 50\%}  &\multicolumn{2}{|c|}{\bf 25\%}  &\multicolumn{2}{|c|}{\bf 50\%}  \\ \cline{3-10}
  & & $\lambda_1$ & $\lambda_2$ & $\lambda_1$ & $\lambda_2$  & $\lambda_1$ & $\lambda_2$  & $\lambda_1$ & $\lambda_2$  \\
 \hline
 \multirow{2}{3em}{\bf MNIST}  & \bf Log. Reg. & 0.32 & 1.28 & 0.64 & 1.28 & 0.64 & 1.28 & - & -  \\ \cline{2-10}
 &\bf  CNN4 & 1.28 & 1.28 & 2.56 & 1.28 & 5.12 & 1.28 & 10.24 & 10.24  \\
 \hline
 \bf CIFAR10 & \bf Resnet18 & 0.024 & 0.008 & 0.032 & 0.004 & 0.032 & 0.01 & - & -  \\
 \hline
\end{tabular}
}
\label{tab:crossval}
\end{table}

\begin{figure}[ht]
 \centering
 \includegraphics[width=8cm]{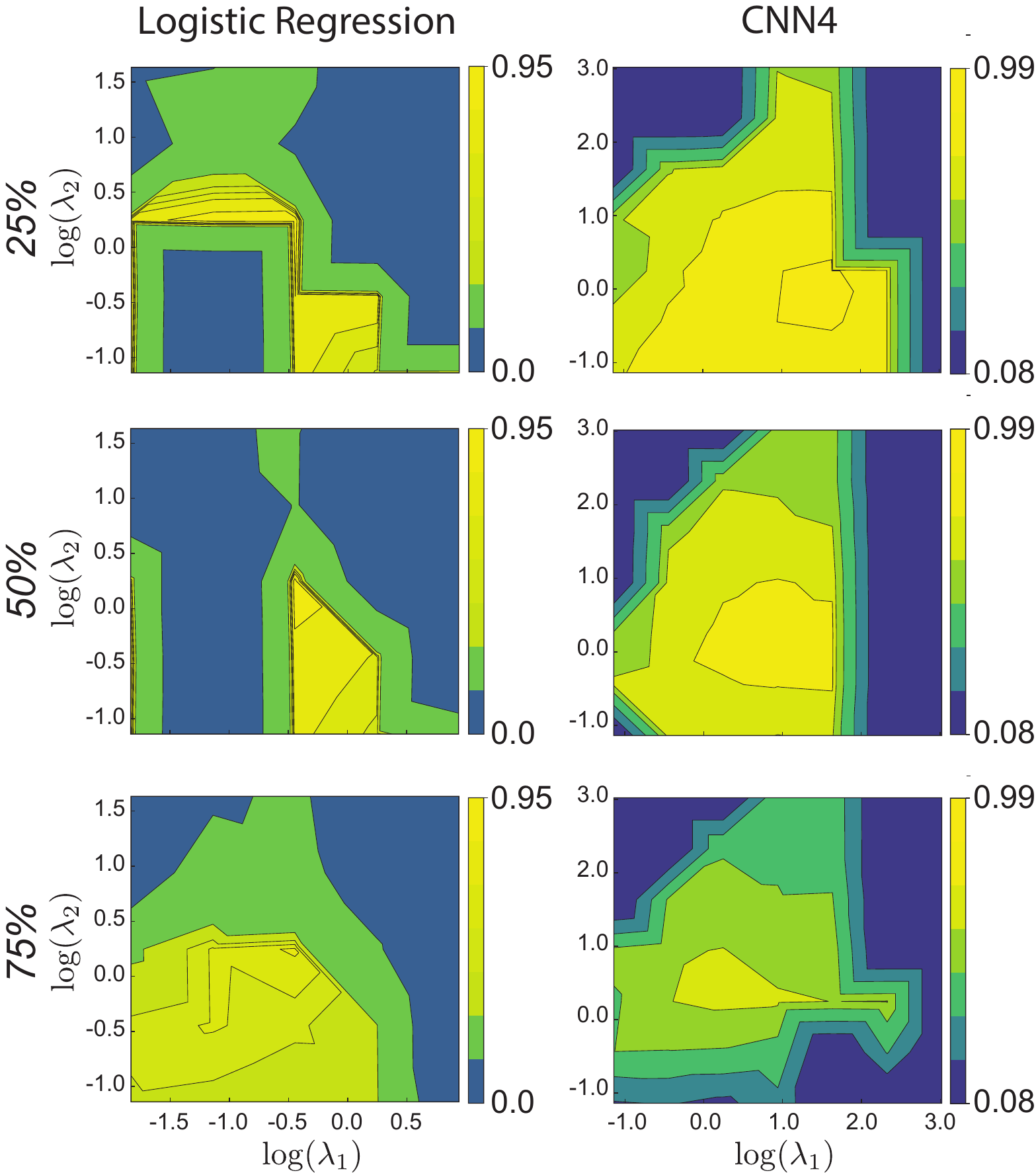}
\caption{\footnotesize{Test accuracy in the grid search for hyper-parameter tuning in MNIST dataset:  $\lambda_1$ and $\lambda_2$ were tuned by cross-validation for various levels of missing entries: 25\%, 50\% and 75\%.}}
\label{fig:hypertuning}
\end{figure}

\begin{figure}[ht]
 \centering
 \includegraphics[width=8cm]{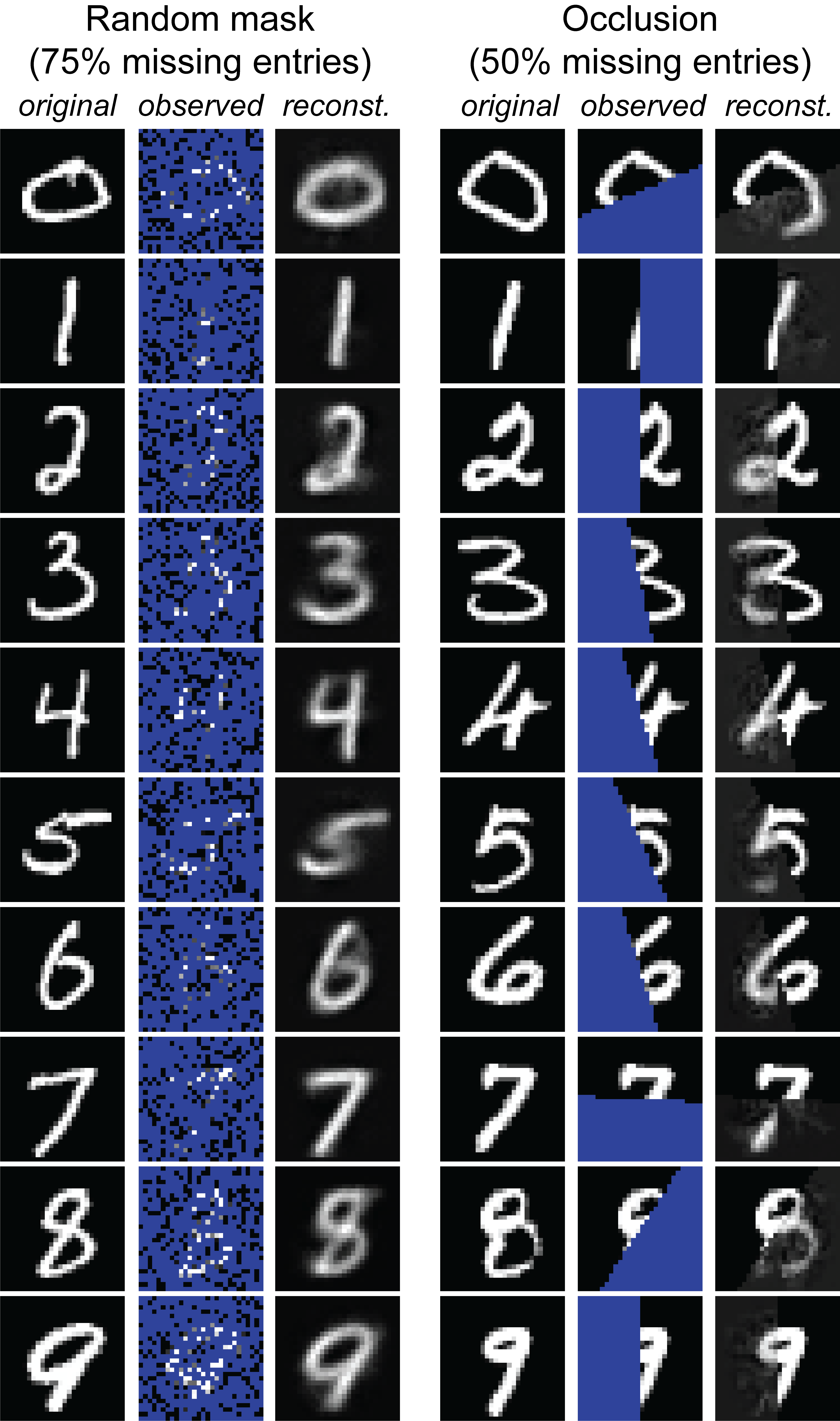}
\caption{\footnotesize{Reconstructions of incomplete test MNIST dataset images by applying our simultaneous classification and coding algorithm with the CNN4 architecture.}}
\label{fig:moreexamplesMNIST}
\end{figure}

\begin{figure}[ht]
 \centering
 \includegraphics[width=8cm]{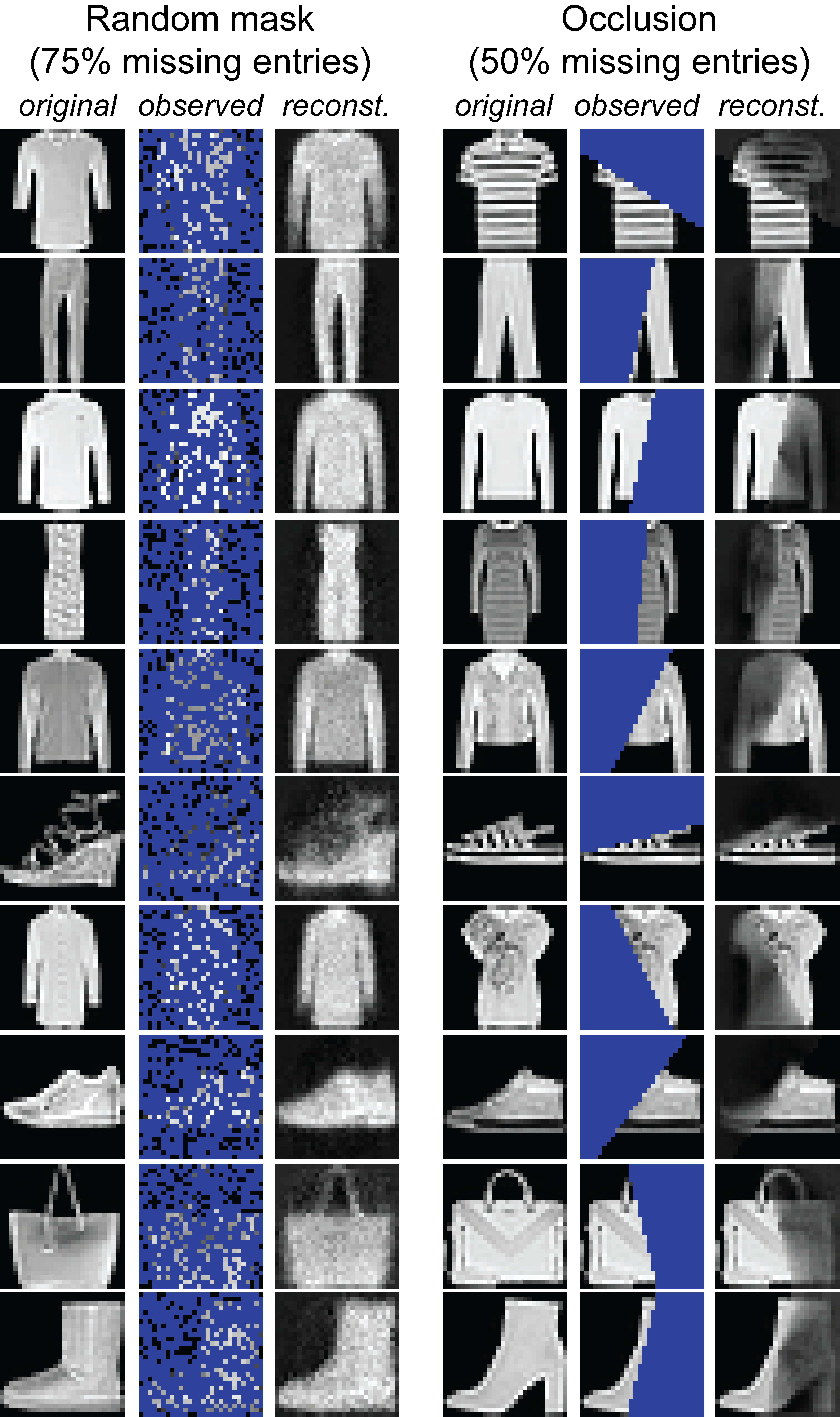}
\caption{\footnotesize{Reconstructions of incomplete test Fashion dataset images by applying our simultaneous classification and coding algorithm with the CNN4 architecture with Batch Normalization (BN).}}
\label{fig:moreexamplesFashion}
\end{figure}

\end{document}